\documentclass{sig-alternate}  

\usepackage[T1]{fontenc}
\usepackage{amsmath} 
\usepackage{booktabs}
\usepackage{xspace}
\usepackage{multirow} 
\usepackage{subfigure}
\usepackage{rotating}
\usepackage{url}
\usepackage{helvet}

\usepackage{algpseudocode}
\usepackage[ruled]{algorithm}

\algnewcommand\algorithmicinput{\textbf{Input:}}
\algnewcommand\algorithmicoutput{\textbf{Output:}}
\algnewcommand\Input{\item[\algorithmicinput]}
\algnewcommand\Output{\item[\algorithmicoutput]}

\DeclareMathOperator*{\fr}{fr}

\newtheorem{prob}{Problem}

\newcommand{\pr}[1]{\left(#1\right)}
\newcommand{\abs}[1]{\left|#1\right|}
\newcommand{\fpr}[1]{\mathopen{}\left(#1\right)}
\newcommand{\enset}[2]{\left\{#1 ,\ldots , #2\right\}}

\newcommand{\freq}[1]{\fr\fpr{#1}}
\newcommand{\prname}[1]{\textsc{#1}}
\newcommand{\dtname}[1]{\emph{#1}}
\newcommand{\fnp}{\textbf{FNP}}
\newcommand{\np}{\textbf{NP}\xspace}
\newcommand{\rp}{\textbf{RP}\xspace}
\newtheorem{theorem}{Theorem}
\newtheorem{lemma}[theorem]{Lemma}

\begin{document}

\conferenceinfo{KDD'09,} {June 28--July 1, 2009, Paris, France.} 
\CopyrightYear{2009}
\crdata{978-1-60558-495-9/09/06} 



\title{Tell Me Something I Don't Know:\\Randomization Strategies for Iterative Data Mining}

\numberofauthors{6}

\author{
Sami Hanhij\"arvi, Markus Ojala, Niko Vuokko, Kai Puolam\"aki, Nikolaj Tatti, Heikki Mannila\\
       \affaddr{Helsinki Institute for Information Technology HIIT}\\
	   \affaddr{Department of Information and Computer Science}\\
	   \affaddr{Helsinki University of Technology, Finland} \\
       \email{fistname.lastname@tkk.fi}
}

\maketitle


\begin{abstract} 
There is a wide variety of data mining methods available, and it is generally useful in exploratory data analysis to use many different methods for the same dataset. This, however, leads to the problem of whether the results found by one method are a reflection of the phenomenon shown by the results of another method, or whether the results depict in some sense unrelated properties of the data. For example, using clustering can give indication of a clear cluster structure, and computing correlations between variables can show that there are many significant correlations in the data. However, it can be the case that the correlations are actually determined by the cluster structure.

In this paper, we consider the problem of randomizing data so that previously discovered patterns or models are taken into account. The randomization methods can be used in iterative data mining. At each step in the data mining process, the randomization produces random samples from the set of data matrices satisfying the already discovered patterns or models. That is, given a data set and some statistics (e.g., cluster centers or co-occurrence counts) of the data, the randomization methods sample data sets having similar values of the given statistics as the original data set. We use Metropolis sampling based on local swaps to achieve this. We describe experiments on real data that demonstrate the usefulness of our approach. Our results indicate that in many cases, the results of, e.g., clustering actually imply the results of, say, frequent pattern discovery.

\end{abstract}


\section{Introduction}

The data mining research of the past 15 years has produced a wide collection of algorithms for exploratory data analysis. 

In this paper we consider a simple question. Suppose we have a dataset $D$, and we first run an analysis using algorithm $A_1$ on $D$, obtaining interesting results $A_1(D)$. Then we use another method $A_2$ on $D$, and again obtain fine results $A_2(D)$. How do we know whether the second result is actually just a consequence of the first, or does it somehow increase our information about the data? So does $A_2(D)$ tell us something we don't know, or is it just rephrasing the result that we already saw when $A_1(D)$ was given to us? 

For example, using clustering can give indication of a clear cluster structure, and computing correlations between variables can show that there are many significant correlations in the data. However, it can be the case that the correlations are actually determined by the cluster structure. 

In this paper, we consider the problem of randomizing data so that previously discovered patterns or models are taken into account. The randomization methods can be used in iterative data mining. At each step in the data mining process, the randomization produces random samples from the set of data matrices satisfying the already discovered patterns or models.

\paragraph{Example}
Consider the toy task of analyzing the dataset shown in
Figure~\ref{fig:exmatrix}.
We can first look at the row and column margins (sums of rows and columns) in the data, observing for instance that the column margins vary between 3 and 7 and the row margins between 1 and 6.

\begin{figure}[t]
\begin{centering}
\begin{tabular}{c*{7}{@{\hspace{0.7mm}}c}}
$A$&$B$&$C$&$D$&$E$&$F$&$G$&$H$ \\
\hline
 1 & 1 & 1 & 0 & 0 & 0 & 1 & 1 \\
 1 & 0 & 1 & 0 & 0 & 1 & 0 & 0 \\
 1 & 1 & 1 & 0 & 0 & 0 & 0 & 1 \\
 1 & 1 & 1 & 0 & 1 & 1 & 0 & 1 \\
 0 & 0 & 0 & 1 & 0 & 0 & 0 & 0 \\
 0 & 0 & 1 & 1 & 1 & 1 & 0 & 0 \\
 0 & 0 & 0 & 1 & 1 & 1 & 1 & 0 \\
 0 & 0 & 1 & 1 & 0 & 0 & 0 & 0 \\
 0 & 0 & 1 & 1 & 1 & 0 & 0 & 1 \\
\hline
\end{tabular}
\par\end{centering}
\caption{Example 0--1 dataset\label{fig:exmatrix} } \end{figure}

The next step in the analysis could be to find the frequent itemsets using minimum frequency of, say, 3. The resulting itemsets are
\begin{align*}
A,B,C,D,E,F,H,AB,AC,AH,BC,BH,CD,CE, \\
CF,CH,DE,EF,ABC,ABH,ACH,BCH,ABCH,
\end{align*}
all with frequency 3 expect itemsets $A,E,F,H,AC,CH$ with frequency 4, itemset $D$ with frequency 5 and itemset $C$ with frequency 7.

The result contains multiple itemsets. Thus it is natural to ask which of them are interesting?

A possible method is to use randomization. Given the row and column margins, we can \cite{Cobb03,Gionis07b} generate datasets that have these margins but are otherwise random. Then we can see how often the frequencies of the sets above are higher in the real data than in the generated datasets, i.e., we can compute empirical $p$-values for the frequencies of the dataset.\footnote{As there are several patterns, we should also correct for multiple testing~\cite{Benjamini95,Holm79}, but for simplicity we omit that consideration in the example.}

By preserving row and column margins in randomization, the frequencies of the itemsets \[AB,BH,ABC,ABH,BCH,ABCH\] are found to be statistically significant with significance level $\alpha=0.05$. Thus even given the information about the row and column margins the frequencies of these itemsets are interesting. The corresponding empirical $p$-values are \[0.044, 0.041, 0.023, 0.004, 0.015, 0.003.\] The other empirical $p$-values are notably larger.

Now the question is, are the significances of the itemsets independent of each other? In other words, are the larger itemsets $ABC,ABH,BCH,ABCH$ statistically significant only because the smaller itemsets $AB,BH$ are?

To answer the question, we would like to compute empirical $p$-values in a way that takes the already known information into account. That is, we would like to constrain the sampling of datasets so that each dataset will have the same frequencies for the itemsets $AB,BH$. This effectively forms a new null hypothesis, which states that the dataset is a random dataset with specific row and column margins and frequencies of the itemsets $AB,BH$. Using this randomization and assessing the statistical significances of the itemsets $ABC,ABH,BCH,ABCH$, we discover that none of these itemsets are statistically significant with this new null hypothesis. The corresponding empirical $p$-values are \[0.229, 0.683, 0.222, 0.170.\] We can therefore conclude that in this example, the frequencies of the smaller itemsets $AB,BH$ and the marginal distributions explain the frequencies of the larger itemsets.

We continue analyzing the dataset by clustering the rows into two clusters using $k$-means. The first four rows are found to form the first cluster and the last five rows the second cluster. To assess the significance of the clustering given the row and column information, we compare the original $k$-means clustering error to clustering errors on randomized datasets where the row and column margins are preserved. We obtain a $p$-value of $0.011$ implying that the dataset contains a clustering structure which is independent of the row and column margins.

Now it would be tempting to conclude that the dataset contains some significant itemsets and an interesting clustering structure. But are the frequent itemsets and the clusters independent of each other? If we again preserve the frequencies of the itemsets $AB,BH$ and the row and column margins in randomization, we get an empirical $p$-value of $0.096$ for the clustering structure of the dataset! That is, the cluster structure does not seem that interesting, given also the data about the two itemsets.

We can also do the reverse: We can preserve the clustering structure in addition to the row and column margins in randomization and test the significances of the frequencies of the itemsets $AB,BH,ABC,ABH,BCH,ABCH$. In this case we obtain the empirical $p$-values \[1.000, 0.239, 1.000, 0.239, 0.239, 0.239,\] respectively, thus implying that the clusters and the frequent itemsets depend on each other.

\paragraph{Contents of the paper}
In this paper we consider the approach described in the above example in a general form. We give a general problem statement, show hardness of the randomization problem and give simple approximate algorithms for the task. We describe experiments on real data that demonstrate the usefulness of our approach. Our results indicate that in many cases, the results of, e.g., clustering actually imply the results of, say, frequent pattern discovery.

The rest of this paper is organized as follows. In Section~\ref{sec:prel} we give the basic definitions. The general problem of randomizing datasets given the information of some other analyses is stated in Section~\ref{sec:probdefs}. Section~\ref{sec:compres} gives some complexity results for this problem, showing that it is in most cases NP-hard. The randomization algorithms based on Metropolis techniques are given in Section~\ref{sec:algs}, while the empirical results are described in Section~\ref{sec:experiments}. Section~\ref{sec:conclusions} is a short conclusion.

\paragraph{Related work}
Randomization is a widely used method in statistics~\cite{Good00, Westfall93}. The main benefit is that the user is releaved from the often difficult, and sometimes impossible, task of defining an analytical distribution for the test statistic. It is sometimes easier to devise a way of sampling from the null hypothesis than to actually define it analytically. And integrating over the analytical distribution, which is needed for the $p$-value calculation, may not be straightforward. MCMC methods are constructed to this purpose, which are also methods for randomization.

\looseness-1 
Based on a wide body of work in ecology (see \cite{Cobb03}), \cite{Gionis07b} discusses the randomization of binary matrices when the size of the matrix as well as the column and row margins are to be maintained. The authors use a Markov chain using local swaps that respect the marginal distributions. The paper~\cite{Ojala08sdm} discusses the same idea for real matrices, where the marginal distributions are no longer single integers for rows or columns, but distributions. They use the Metropolis-Hastings method to define transition probabilities for their local changes to maintain the desired marginal distributions. We take influence from these papers to produce our methods.

Our concepts are closely related to the concept on non-derivability~\cite{Calders07}. An itemset $X$ is derivable if we can reduce the exact support of $X$ from the supports of its sub-itemsets. For example, if the support of $A$ is $0$, then the support of $AB$ must also be $0$. This means that if we preserve the supports of the sub-itemsets, then the support of $X$ will be constant, thus its $p$-value will be 1. Hence, our method can be seen as a generalization of non-derivability in which we can still remove insignificant $X$ even though we do not know the exact support of $X$. 

A closely related idea for iterative knowledge discovery is presented in~\cite{Jensen91}, where randomization is used in model selection to test if a candidate model is significantly better than the current model. The current model is replaced by a significantly better candidate model and the process is repeated, effectively carrying out model selection iteratively.

The pattern ordering problem considered in \cite{Mielikainen03} tries to order a collection of patterns so that each pattern gives as much information about the data as possible, given the earlier patterns. However, in that approach there is no consideration of randomization or significance, and the approach is only applicable to simple patterns. A related idea for closed itemsets, where statistical significance is used to order the patterns, is presented in~\cite{Gallo07}.
For another type of approach to significance of patterns, see \cite{Jaroszewicz08,Webb07,Webb08}.

\section{Preliminaries} \label{sec:prel}

In this section, we describe how randomization approach is used in significance testing, give the definition of empirical $p$-values and discuss the method by Besag and Clifford for calculating MCMC $p$-values. First, however, we introduce the notation used in the paper. 

\subsection{Notation}

Let $D$ be a 0--1 dataset with $m$ rows and $n$ columns. The approach we describe is not limited to 0--1 data, but for simplicity we consider just such data in the paper. The rows of $D$ correspond to \emph{transactions} and the columns to \emph{attributes}. The notation $D_{tx}$ refers to the element at row $t$ and column $x$ in the matrix $D$. An \emph{itemset} $X$ is a subset of the attributes, i.e., $X\subset\{1,\ldots,n\}$. A transaction $t$ \emph{covers} an itemset $X$ if $D_{tx}=1$ for all $x\in X$. The \emph{frequency} of an itemset $X$ in the dataset $D$ is the number of transactions $t$ that cover the itemset $X$, denoted by $\fr(D;X)$. A \emph{family of itemsets} $\mathcal{F}$ is a set of itemsets, $\mathcal{F}=\{X_1,\ldots,X_k\}$. A \emph{clustering} $\mathcal{C}$ of the rows of $D$ is a partition of the set $\{1,\ldots,m\}$. The row and column \emph{margins} of a dataset $D$ are the row and column sums of $D$.

\subsection{Randomization Approach}

Consider a 0--1 dataset $D$ with $m$ rows and $n$ columns. Assume that some data mining task, such as frequent itemset mining, is performed on $D$. Let $\mathcal{S}(D)$ be the result of the data mining task. We assume that the result $\mathcal{S}(D)$ can be described with a single number, and we call such a result the \emph{structural measure} of $D$. It can be, e.g., the frequency of a given itemset, the number of frequent itemsets or the clustering error of the dataset. Any measure can be used as long as larger (or smaller) values mean stronger presence of the structure.

To assess whether the result $\mathcal{S}(D)$ is explained by certain characteristics of the original dataset, we generate random $m\times n$ sized 0--1 datasets $\hat{D}$, which share the given characteristics with $D$, and compare the original result $\mathcal{S}(D)$ with the results $\mathcal{S}(\hat{D})$ on the randomized datasets. We can, e.g., preserve the row and column margins in randomization and assess the number of frequent itemsets. 

\subsection{Empirical p-Values}
Let $\hat{\mathcal{D}}=\{\hat{D}_1,\ldots,\hat{D}_k\}$ be a collection of independent randomized versions of the original dataset $D$. The one-tailed \emph{empirical $p$-value} of $\mathcal{S}(D)$ for $\mathcal{S}(D)$ being large is 
\begin{equation}
\frac{|\{\hat{D}\in\hat{\mathcal{D}}\ |\ \mathcal{S}(\hat{D}) \geq \mathcal{S}(D)\}|+1}{k+1},
\end{equation}
This gives the fraction of randomized datasets whose structural measure is larger than the original $\mathcal{S}(D)$. The one-tailed empirical $p$-value when small values of $\mathcal{S}(D)$ are interesting, and the two-tailed empirical $p$-value are defined similarly. If the obtained $p$-value, adjusted for multiplicity if needed, is less than a given threshold $\alpha$, say, $\alpha=0.05$, we can regard the result to be independent of the characteristics preserved in randomization.

\subsection{MCMC p-Values}

We will use Markov chain Monte Carlo methods to produce the randomized datasets. The samples produced by Markov chains are generally not independent thus breaking the validity of the empirical $p$-value. However, we will use the approach by Besag and Clifford~\cite{Besag89} to guarantee the \emph{exchangeability} of the samples. In the approach the chain is started from $D$ and run backwards for $K$ steps to produce a new starting state $\hat{D}_0$. Each randomized dataset $\hat{D}$ is produced by starting a new chain from $\hat{D}_0$ and running the chain $K$ steps forwards. This produces an exchangeable set of samples $\{D,\hat{D}_1,\ldots,\hat{D}_k\}$ which ensures the validity of the empirical $p$-value regardless of the independence. If the samples are dependent, we obtain just more conservative $p$-values, see~\cite{Besag04,Besag89} for more details. Our methods turn out to be time-reversible, i.e., running the chain backwards is the same as running the chain forwards. 

\section{Problem Statement} \label{sec:probdefs}

In this section we formulate the general 
specific randomization problems discussed in the introduction. 

\subsection{General Problem Statements}
In randomization we want to preserve certain characteristics of the original 0--1 dataset $D$, e.g., the row and column margins. Let $f(D)$ be a \emph{statistics function} which calculates these statistics of the dataset $D$. To measure the similarity between two datasets $D$ and $\hat{D}$ in the corresponding statistics $f(D)$ and $f(\hat{D})$, we define a \emph{difference measure} $h(f(D),f(\hat{D}))$ which is a positive function between two sets of statistics $f(D)$ and $f(\hat{D})$, where $D$ and $\hat{D}$ are 0--1 datasets with the same size. Any difference measure can be used as long as a smaller difference means that the statistics $f(D)$ and $f(\hat{D})$ are closer to each other and a zero difference that the statistics are equal.

We introduce two general randomization problems. In the first problem the statistics are preserved exactly where as in the second problem datasets with small difference in the statistics are sampled with high probability. The problem statements are as follows.
\begin{prob}[ExactRand]\label{prob:exactrand}
Given a 0--1 dataset $D$ and a statistics function $f$, generate a dataset $\hat{D}$ chosen independently and uniformly from the set of 0--1 datasets having the same size and the same statistics as $D$, i.e., $f(D)=f(\hat{D})$.
\end{prob}
\begin{prob}[SoftRand]\label{prob:softrand}
Given a 0--1 dataset $D$, a statistics function $f$, a difference measure $h(f(D),f(\hat{D}))$ and a scaling constant $w>0$, generate a dataset $\hat{D}$ chosen with a probability \[p\propto\exp\{-w h(f(D),f(\hat{D}))\}\] from all 0--1 datasets having the same size as $D$. 
\end{prob}
Note that when $w=\infty$, \prname{SoftRand} reduces to \prname{ExactRand}.

\subsection{Specific Problem Statements} \label{sec:probdefs:spec}
Next we give the problem statements of the three specific randomization tasks discussed in the introduction. The problems \prname{Margins}, \prname{ClusterMargins} and \prname{ItemsetMargins} are examples of the problem \prname{ExactRand}.
\begin{prob}[Margins] \label{prob:swap}
Given a 0--1 dataset $D$, generate a dataset $\hat{D}$ chosen independently and uniformly from the set of 0--1 datasets having the same row and column margins as the dataset $D$.
\end{prob}
\begin{prob}[ClusterMargins] \label{prob:cluster-swap}
Given a 0--1 dataset $D$ and a clustering $\mathcal{C}$ of the rows of $D$, generate a dataset $\hat{D}$ chosen independently and uniformly from the set of 0--1 datasets having the same row and column margins as well as the same cluster centers and variances for each cluster in $\mathcal{C}$ as the dataset $D$.
\end{prob}
\begin{prob}[ItemsetMargins] \label{prob:itemset-swap-hard}
Given a 0--1 dataset $D$ and a family of itemsets $\mathcal{F}$, generate a dataset $\hat{D}$ chosen independently and uniformly from the set of 0--1 datasets having the same row and column margins as well as the same frequencies for the itemsets in $\mathcal{F}$ as the dataset $D$.
\end{prob}

It turns out that the problem \prname{ItemsetMargins} is 
much harder than \prname{Margins} and \prname{ClusterMargins}, see Section~\ref{sec:compres}. Thus we introduce \prname{SoftRand} version of \prname{ItemsetMargins}. Let $D$ be the original dataset, $\hat{D}$ a randomized dataset and $\mathcal{F}$ a family of itemsets which we are trying to preserve in randomization. We define the difference in the itemset frequencies of $\mathcal{F}$ between the datasets $D$ and $\hat{D}$ as
\begin{equation} \label{eq:diff-itemset-freqs}
h_{\mathcal{F}}(D,\hat{D}) = \sum_{X\in\mathcal{F}} |\fr(D;X)-\fr(\hat{D};X)|, 
\end{equation}
where we have combined the statistics function directly into the difference measure. 
\begin{prob}[ItemsetMarginsSoft] \label{prob:itemset-swap-soft}
Given a binary dataset $D$, a family of itemsets $\mathcal{F}$ and a scaling constant $w>0$, generate a dataset $\hat{D}$ chosen with a probability \[p\propto\exp\{-w h_{\mathcal{F}}(D,\hat{D})\}\] from the set of 0--1 datasets having the same row and column margins as the dataset $D$.
\end{prob}

\section{Complexity Results} \label{sec:compres}

In this section we prove that the \prname{ItemsetMargins} variant defined in Section~\ref{sec:probdefs:spec} is intractable in general case. We will do this by reducing the \prname{HamiltonCycle} to \prname{ItemsetMargins}. This negative result shows the need for \prname{ItemsetMarginsSoft} where we allow some variation in the frequencies.

\begin{theorem}
Assume that there is a random polynomial algorithm for \prname{ItemsetMargins}.
Then $\rp = \np$ even if the algorithm is provided with an example of such
dataset.
\label{thr:complexity}
\end{theorem}

A random polynomial algorithm is 
a Turing machine that is bound\-ed by a polynomial time with an oracle producing random
independent numbers. A language is said to be in \rp if there is a
non-deterministic Turing machine such that, given a 'yes'-instance, at least
half of the computational paths end up with 'yes'. It is easy to see that $\rp
\subseteq \np$, and the usual conjecture is that $\rp \neq \np$.

Our reduction is based on \emph{Hamiltonian cycles}. A Hamiltonian cycle is a
cycle in a graph such that every node of the graph is visited only once.
Alternatively we can see the cycle as a permutation of the nodes such that
adjacent nodes (including the first and the last nodes) are connected. The
problem \prname{HamiltonCycle} is \fnp-complete~\cite{Papadimitriou94}.

To prove the main result we will need a couple of lemmae. The first lemma
states that we can connect Hamiltonian cycles and the datasets satisfying
some specific constraints.

\begin{lemma}
Assume that we are given a graph $G$. There are column and row margins, itemset frequency constraints, and a function $m$ that maps a dataset satisfying the constraints into a Hamiltonian cycle of $G$.
\label{lem:map}
\end{lemma}

\begin{proof}
Assume that we are given a graph $G$ with $M$ nodes and $N$ edges. Our goal is
to construct appropriate constraints. At first, we will focus only on
itemsets and column margins. The row margins will be discussed later. In our
construction we will have $6$ different attribute groups.

\looseness-1
The first attribute group $O = \enset{o_1}{o_M}$ contains $M$ attributes.  We
impose the frequencies $\freq{o_i} = 1$, $\freq{o_io_j} = 0$ for $i, j = 1,
\ldots, M$, $i \neq j$. These frequencies will force that for each $o_i$
there is a row on which the attribute value is $1$ and that there are no other
active $o_j$ on that row, i.e., 
if we stack the rows into a matrix
we will have a permutation matrix.

The second group $V = \enset{v_1}{v_M}$ is similar to the first,
that is, we have $\freq{v_i} = 1$ and $\freq{v_iv_j} = 0$. This construction
gives us also a permutation matrix. Also note that for each $o_i$ there is a
\emph{unique} $v_j$ such that there is a row in which both $o_i$ and $v_j$ are
present. This allows us to define a permutation of $\sigma(i) = j$. The main
idea is that $\sigma$ represents the Hamiltonian cycle. This can be achieved if
we can force that the nodes in $G$ corresponding to $\sigma(i)$ and $\sigma(i +
1)$ are connected. We will achieve this with the rest of the attributes.

The third group $B = \enset{b_1}{b_M}$ contains the attributes
satisfying $b_i = o_i \lor o_{i + 1}$. This is done by imposing the
itemsets $\freq{b_i} = 2$, $\freq{b_io_i} = 1$, and $\freq{b_io_{i + 1}} = 1$
for $i = 1, \ldots, M - 1$ and $\freq{b_M} = 2$, $\freq{b_Mo_M} = 1$, and
$\freq{b_Mo_1} = 1$. We see that the attribute $b_i$ is present on two rows
and the rows are precisely those in $o_i$ and $o_{i + 1}$ are present.

Our fourth group of items resembles greatly the third group. We denote the
group by $C = \enset{c_1}{c_{M(M-1)/2}}$. The group contains $M(M-1)/2$ attributes,
where each $c_i$ correspond to a pair of nodes $(j, k)$ in the graph $G$. We
define the frequencies such that $c_i$ is exactly $v_j \lor v_k$.  This is done
by setting $\freq{c_i} = 2$ and $\freq{c_iv_j} = \freq{c_iv_k} = 1$.

Now let us consider what happens if $\sigma(i)$ and $\sigma(i + 1)$ are not
connected in $G$. There is a $c_j$ corresponding to the node pair
$\pr{\sigma(i), \sigma(i + 1)}$ such that the columns $b_i$ and $c_j$ are
equivalent. In other words, we have $\freq{b_ic_j} = 2$. Thus, to guarantee
that $\sigma$ is indeed a valid Hamiltonian cycle we need to make sure that for
all $b_i$ and $c_j$ corresponding to unconnected pairs in $G$ we have
$\freq{b_ic_j} < 2$.

Since we do not have inequality constraints in our setup we will have to
simulate it. Let $L = M(M-1)/2 - N$ be the number of pairs of unconnected
nodes.  We define the fifth set of attributes, denoted by $S =
\enset{s_1}{s_{LM}}$, to contain $LM$ attributes, one $s_{ij}$ for each pair
$\pr{b_i, c_j}$, where $c_j$ represent a pair of unconnected nodes. The needed
inequality constraint is simulated by setting $\freq{s_{ij}} = 1$,
$\freq{s_{ij}b_i} = 1$, and $\freq{s_{ij}c_j} = 0$. This construction forces
$b_i$ and $c_j$ to be unequal so that the frequency $\freq{b_ic_j} \neq
2$.

We see now that using the defined itemsets and column margins from all $5$
groups forces the permutation $\sigma$ to be a valid Hamiltonian cycle.

Let us now turn the attention to the row margins. The number of 1s in a single row is
as follows: one 1 from group $O$, one 1 from group $V$, two 1s from group $B$
and $M - 1$ 1s from group $C$. Group $S$ is problematic since the number of $1$
varies per row. We remedy this by defining the sixth group $E =
\enset{e_1}{e_{LM}}$. This group contains $LM$ attributes. We set $e_i$ to be
the negation of $s_i$. We achieve this by setting $\freq{e_i} = M - 1$ and
$\freq{s_ie_i} = 0$. Now the common number of 1s on single row in group $S$ and
$E$ is $LM$. Hence by setting all the row margins to $1 + 1 + 2 + M - 1 + LM$
we have created a set of constraints such that a dataset satisfying the
constraints corresponds to a Hamiltonian cycle.
\end{proof}

Our second lemma states that for each Hamiltonian cycle in a given graph there
is exactly same number of datasets satisfying the constraints.

\begin{lemma}
Given a graph $G$ and a Hamiltonian cycle $H$. Let $m$ be the map
given in Lemma~\ref{lem:map}. Define $\mathcal{X} = \{D; m(D)$ $= H\}$
to be the datasets which $m$ maps into $H$. Then $\abs{\mathcal{X}}$
is a constant not depending on the Hamiltonian cycle $H$.
\label{lem:const}
\end{lemma}

\begin{proof}
To prove the lemma first note that a dataset satisfying the constraints has $M$
\emph{unique} rows. Thus there are $M!$ datasets 
obtained by 
permuting the rows. Assume now that the group $O$ has the shape of the identity
matrix. There are exactly $2M$ different permutations for a given Hamiltonian
cycle, i.e., 
exactly $2M$ different configurations for $V$. Note
that the groups $B$ and $C$ are determined completely by the groups $O$ and $V$
and that the group $E$ is determined by the group $S$. The theorem is proved if
we can show that there are a constant number of configurations for $S$.

Note that there are $2^U$ configurations for $S$ where $U$ is the number of
$s_{ij}$ for which $\freq{b_ic_j} = 0$. The frequency $\freq{b_ic_j}$ is $0$
when both the $\sigma(i)$ and $\sigma(i + 1)$ do not contain the nodes
corresponding to the $c_j$. For a fixed $j$ there are $M - 4$ of such $s_{ij}$.
Hence we have $U = (M - 4)L$.

Combining the numbers together we have that for a fixed Hamiltonian cycle we
have exactly $M!2M2^{(M - 4)L}$ datasets which is a constant. This proves the theorem.
\end{proof}

\begin{proof}[of Theorem~\ref{thr:complexity}]
We start the proof by considering a related \np-complete problem called
\prname{SecondHamilton}~\cite{Papadimitriou94}. In this problem we are given a Hamiltonian
graph, a Hamiltonian cycle, and we are asked if there is a second Hamiltonian
cycle different from the given one.

Now let us consider a problem \prname{RandomHamilton} having the same input as
\prname{SecondHamilton} in which we are asked to
give a random Hamiltonian cycle of the graph. Assume that we have a random
polynomial algorithm $H$ for this problem. Now consider the problem of
\prname{SecondHamilton}.  We can replace $H$ with a non-deterministic machine
by considering all the possible random outputs of the oracle. Then we compare
the random Hamiltonian cycle returned by $H$ with a given Hamiltonian cycle. If
the cycles differ, then we return 'yes', otherwise 'no. This means that if there
is another cycle in the graph, then at least $1/2$ of the computation paths ends
up with 'yes'. Thus we have shown that \prname{SecondHamilton} is in \rp. But
\prname{SecondHamilton} is \np-complete and this proves that $\np = \rp$.

Now assume that there is a polynomial algorithm for \prname{ItemsetMargins}. Such an algorithm is given a set of constraints and an example dataset satisfying the constraints.  By using the construction given in Lemma~\ref{lem:map} we can use that algorithm for sampling Hamiltonian cycles in polynomial time. The given Hamiltonian cycle can be also easily transformed into a dataset needed for the input of the algorithm.

The only thing we need to show is that this reduction produces Hamiltonian
cycles from the uniform distribution.  Since we have assumed that the datasets
are coming from the uniform distribution, it suffices to prove that there are
same number of datasets for each Hamiltonian cycle in a fixed graph. But this
is exactly the statement of Lemma~\ref{lem:const}.
\end{proof}

\section{Algorithms} \label{sec:algs}

Next we introduce algorithms for solving the problems \prname{Margins}, \prname{ClusterMargins} and \prname{ItemsetMarginsSoft}. First we introduce the method by Gionis \emph{et al.}~\cite{Gionis07b} for solving \prname{Margins}. Our methods for \prname{ClusterMargins} and \prname{ItemsetMarginsSoft} extend the method for solving \prname{Margins}.

\subsection{Preserving Row and Column Margins}

The method for producing random datasets $\hat{D}$ having the same row and column margins as the original dataset $D$ is based on \emph{swaps}. In each swap two rows $s,t$ and two columns $x,y$ are selected such that $D_{sx}=D_{ty}=1$ and $D_{sy}=D_{tx}=0$. In a swap the four elements are swapped as shown in Figure~\ref{fig:swap}. A swap preserves the row and columns sums. A randomized dataset $\hat{D}$ is produced by performing $K$ attempts of swaps. This is given in Algorithm~\ref{alg:swap}. The existence of self-loops guarantees that the stationary distribution is uniform, see~\cite{Gionis07b} for more details.

\begin{figure}
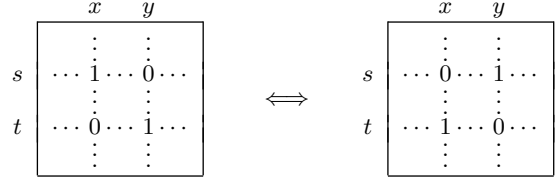

\centering
\hspace{\stretch{1}}
\begin{tabular}{c|*{5}{c@{\hspace{0.5mm}}}|}
\multicolumn{2}{c}{} & $x$ && $y$ & \multicolumn{1}{c}{} \\
\cline{2-6}
 &&$\vdots$&&$\vdots$&  \vspace{-1mm}\\
$s$ & $\cdots$ & 1 & $\cdots$ & 0 & $\cdots$ \vspace{-2mm} \\
 &&$\vdots$&&$\vdots$& \vspace{-1mm} \\
$t$ & $\cdots$ & 0 & $\cdots$ & 1 & $\cdots$ \vspace{-2mm} \\
 &&$\vdots$&&$\vdots$&\\
\cline{2-6}
\end{tabular}
\hspace{\stretch{1}}
\hspace{3mm}\raisebox{-2mm}{$\iff$} 
\hspace{\stretch{1}}
\begin{tabular}{c|*{5}{c@{\hspace{0.5mm}}}|}
\multicolumn{2}{c}{} & $x$ && $y$ & \multicolumn{1}{c}{} \\
\cline{2-6}
 &&$\vdots$&&$\vdots$&  \vspace{-1mm}\\
$s$ & $\cdots$ & 0 & $\cdots$ & 1 & $\cdots$ \vspace{-2mm} \\
 &&$\vdots$&&$\vdots$& \vspace{-1mm} \\
$t$ & $\cdots$ & 1 & $\cdots$ & 0 & $\cdots$ \vspace{-2mm} \\
 &&$\vdots$&&$\vdots$&\\
\cline{2-6}
\end{tabular}
\hspace{3mm}
\hspace{\stretch{1}}
\caption{A swap in a 0--1 matrix.}
\label{fig:swap}
\end{figure}

\begin{algorithm}
\caption{Swap}\label{alg:swap}
\begin{algorithmic}[1]
\Input Dataset $D$, num. of swap attempts $K$
\Output Randomized dataset $\hat{D}$
\State $\hat{D} \gets D$
\For{$i \gets 1$ to $K$}
\State Pick $s,t$ and $x,y$ such that $\hat{D}_{sx}=1$, $\hat{D}_{ty}=1$
\If{$\hat{D}_{sy}=0$ and $\hat{D}_{tx}=0$}
\State $\hat{D} \gets $ swapped version of $\hat{D}$
\EndIf
\EndFor
\State \Return $\hat{D}$
\end{algorithmic}
\end{algorithm}

\subsection{Preserving Clustering Structure}

 To obtain an algorithm for the problem \prname{ClusterMargins}, we modify Algorithm~\ref{alg:swap} to preserve the given clustering $\mathcal{C}$. At each step we pick a cluster $C\in\mathcal{C}$ and attempt a swap inside that cluster. When the rows $s,t$ belong to the same cluster, the cluster centers and variances do not change. The pseudocode of this approach is given in Algorithm~\ref{alg:cluster-swap}.

\begin{algorithm}
\caption{Cluster-Swap}\label{alg:cluster-swap}
\begin{algorithmic}[1]
\Input Dataset $D$, partition $\mathcal{C}$ of $D$, num. of swap attempts~$K$
\Output Randomized dataset $\hat{D}$
\State $\hat{D} \gets D$
\For{$i \gets 1$ to $K$}
\State Pick a cluster $C\in\mathcal{C}$
\State Pick $s,t\in C$ and $x,y$ such that $\hat{D}_{sx}=1$, $\hat{D}_{ty}=1$
\If{$\hat{D}_{sy}=0$ and $\hat{D}_{tx}=0$}
\State $\hat{D} \gets $ swapped version of $\hat{D}$
\EndIf
\EndFor
\State \Return $\hat{D}$
\end{algorithmic}
\end{algorithm}

\subsection{Preserving Itemset Frequencies}

As was mentioned in Section~\ref{sec:probdefs}, preserving itemset frequencies exactly in general is hard. Thus we give an algorithm for solving \prname{ItemsetMarginsSoft} where the itemset frequencies are preserved approximately. We use the Metropolis algorithm~\cite{Metropolis53} to produce random samples from the probability distribution 
\[
\pi(\hat{D}| D,\mathcal{F}) \propto \exp\{-w h_{\mathcal{F}}(D,\hat{D})\},
\]
where $h_{\mathcal{F}}$ is defined in Equation~\eqref{eq:diff-itemset-freqs}. At each step in the Metropolis algorithm, a proposal modification $D'$ of the current state $\hat{D}$ is formed. The proposal is accepted as the new state with a probability $\min(1,\pi(D')/\pi(\hat{D}))$. A direct implementation of the Metropolis algorithm with swaps is given in Algorithm~\ref{alg:itemset-swap}.

\begin{algorithm}
\caption{Itemset-Swap}\label{alg:itemset-swap}
\begin{algorithmic}[1]
\Input Dataset $D$, itemsets $\mathcal{F}$, num. of swap attempts $K$
\Output Randomized dataset $\hat{D}$
\State $\hat{D} \gets D$
\For{$i \gets 1$ to $K$}
\State Pick $s,t$ and $x,y$ such that $\hat{D}_{sx}=1$, $\hat{D}_{ty}=1$
\If{$\hat{D}_{sy}=0$ and $\hat{D}_{tx}=0$}
\State $D' \gets $ swapped version of $\hat{D}$
\State $a \gets$ Uniform(0,1)
\If{$a < \exp\{-w(h_{\mathcal{F}}(D,D')-h_{\mathcal{F}}(D,\hat{D}))\}$}
\State $\hat{D} \gets D'$
\EndIf
\EndIf
\EndFor
\State \Return $\hat{D}$
\end{algorithmic}
\end{algorithm}

The same approach can be used to solve the problem \prname{Soft\-Rand} in general. However, if too many characteristics are preserved at the same time, the chain may not mix well enough. In such cases parallel tempering can be used to overcome the problem~\cite{Geyer91}. In our experiments, we consider only a reasonable amount of small itemsets as constraints.

\section{Experiments}\label{sec:experiments}

In this section, we carry out several data mining experiments to demonstrate the use of our framework. We will first present a method to automatically analyze the convergence of the Markov chain. The method is then used in the subsequent experiments with two real data sets.

\subsection{Setup for the Experiments}

We use two real data sets: \dtname{Paleo} and \dtname{Courses}, whose basic characteristics are presented in Table~\ref{tab:datasets}. The \dtname{Paleo}\footnote{NOW public release 030717 available from~\cite{Fortelius05now}.} data contains paleontological absence/presence information about species in excavation sites~\cite{Fortelius05now}. The \dtname{Courses} data set records the courses individual students have taken in the Department of Computer Science in Helsinki University of Technology. 

\begin{table}
\begin{centering}
\begin{footnotesize} \begin{tabular}{lrrrr}
\toprule 
Data set & \# of rows & \# of cols & \# of 1's & density\\
\midrule
\dtname{Paleo} & 124 & 139 & 1978 & 11.48\%\\
\dtname{Courses} & 2405 & 5021 & 65152 & 0.54\%\\
\bottomrule
\end{tabular}\end{footnotesize} 
\par\end{centering}

\caption{Basic characteristics of the datasets\label{tab:datasets}.}

\end{table}

We used the method by Webb~\cite{Webb07} to ensure correct treatment of multiple hypothesis. We split both data sets randomly half on rows, and used the other part for mining frequent itemsets of size 2 and 3. The other part was used for randomization, i.e., $p$-value calculation.\footnote{We also tried not using Webb's method and use the complete data sets for mining itemsets as well as randomization. However, the results did not differ significantly.}

The \dtname{Paleo} data set was mined with minimum support 4, and \dtname{Courses} with 200. It turned out that only 41 columns of \dtname{Courses} are covered by the found frequent sets. We choose to ignore the columns of \dtname{Courses} that are not present in any of the frequent sets.  We chose to consider only the frequent sets of size 2 and 3, since it may be fairly easy to understand the co-occurrence of 2 or 3 variables, but understanding itemsets of larger size is increasingly difficult.

When calculating the $p$-values of itemsets, we use the support as the test statistic, where a higher value is more extreme, i.e., more interesting. We always sample 1000 matrices and use them for $p$-value calculation. The unadjusted, raw $p$-values are always adjusted for multiplicity with Benjamini-Hochberg method~\cite{Benjamini95}, which controls the false discovery rate (FDR). The level of FDR is set to 0.05, and thus, if the $p$-value of a pattern does not exceed this threshold value, the pattern is considered statistically significant. 

We also cluster the data sets in some of the experiments. We use the traditional $k$-means with 3 clusters for \dtname{Paleo} and 20 clusters for \dtname{Courses}. Since $k$-means may get stuck in local optima, we run $k$-means 10 times and use the clustering with the smallest error. When calculating the $p$-value of a clustering, our test statistic is the sum of the square of the $L_{2}$ distances between rows and their respective centroids. This is almost the same as optimized by $k$-means, but we use this slightly modified version since we know the \prname{ClusterMargins} maintains the value of this test statistic.

We use the Algorithm~\ref{alg:itemset-swap} to solve \prname{ItemsetMarginsSoft}, with $w=4$. That is, a swap which would increase the total error in itemset frequencies by one is accepted with probability $\exp(-4)\approx0.018$.

\subsection{Convergence Analysis}

In our experiments we used an automatic convergence analysis method that
is based on the distance between the original and the randomized matrix. 
As a distance we used the square of the Frobenius norm between these two
matrices~\cite{Ojala08sdm}. With 0--1 matrices, this is essentially the number of
cells where the two matrices differ. When swapping, this distance usually first
increases rapidly, until after some number of swaps it tends to settle
around a certain value. The intuition is that when the distance does not change
much, the Markov chain is considered to have converged. 

\looseness-1
We calculate the number of swap attempts $K$ needed as follows. We first assign to $K$
the amount of ones in the matrix, as it gives an indication of the number of
swaps possible in that matrix, given that the matrix is sparse. We then
randomize the original matrix separately 5 times with $K$ swaps and calculate
the mean of the matrix distances between the randomized data and the original
data. If the mean distance with the current $K$ has changed less than 1\%
with respect to the previous step, we consider the chain to have converged,
stop iterating, and use the final value of $K$ in the sampling. Otherwise, we
multiply $K$ by $2$ and repeat the step.
 
As an example, we run the automated convergence analysis on \dtname{Courses}
with 40 frequent sets of size 2 as constraints. The development of the matrix
distances is given in Figure~\ref{fig:convergence}. 

\begin{figure}
\begin{centering}
\includegraphics{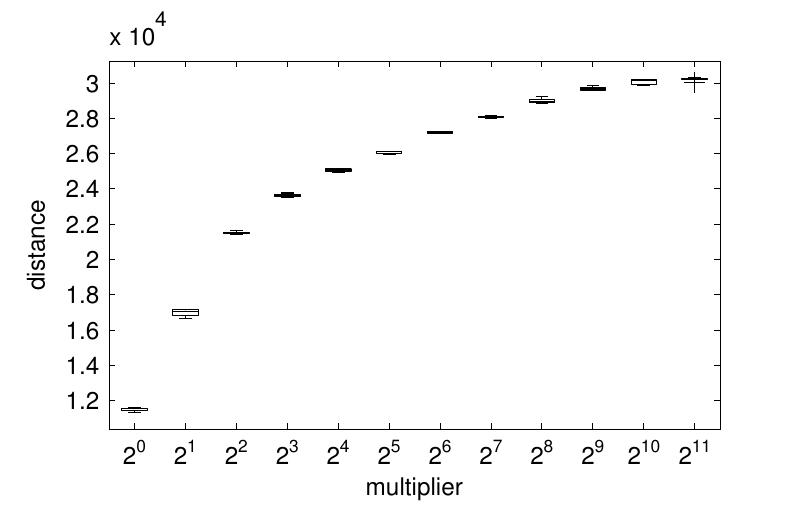}
\par\end{centering}
\caption{Boxplot of distances between the original matrix and 5 \prname{ItemsetMarginsSoft} randomized matrices of \dtname{Courses} data with different number of swap attempts. The x-axis depicts the multiplier to use with the number of ones in \prname{Courses} to get the number of swap attempts.\label{fig:convergence}}
\end{figure}

Note that since we use backward forward sampling, the chain is first traversed
backwards with $K$ swap attempts, and the resulting matrix is used as a basis
for further randomizations. Eventually, the number of swap attempts from the
original matrix to a matrix sample is $2K$.

\subsection{Significant Itemsets of Size 2}

In the first experiment, we are interested in the covariances between the columns when constraining the randomization in different ways. Therefore, we focus on all possible itemsets of size 2, and randomize with \prname{Margins}, \prname{ItemsetMarginsSoft} and \prname{ClusterMargins}. For the constraints of \prname{ItemsetMarginsSoft} we used itemsets of size $2$ containing adjacent items. Notice that we did not use Webb's method here since we consider all possible patterns.

The execution times to produce a single random matrix with different constraints and data sets are listed in Table~\ref{tab:times}.

\begin{table}
\begin{centering}
\begin{footnotesize} \begin{tabular}{lrrrr}
\toprule 
Data set & M & IMS & CM \\
\midrule
\dtname{Paleo} & 4.9 & 360 & 15 \\
\dtname{Courses} & 64 & 5200 & 200 \\
\bottomrule
\end{tabular}\end{footnotesize} 
\par\end{centering}

\caption{Average execution times in milliseconds to produce a single random matrix with \prname{Margins} (M), \prname{ItemsetMarginsSoft} (IM) and \prname{ClusterMargins} (CM). The implementations are written in Java and the randomizations run on a 2.5GHz machine.\label{tab:times}}

\end{table}

Table~\ref{tab:pairs_numbers} depicts several interesting contingency tables of the significance of itemsets in two randomizations at a time. We also measured the significance of the clustering result. Using the defined test statistic, the clustering was statistically significant in \prname{Margins} and \prname{ItemsetMarginsSoft}, but naturally not in \prname{ClusterMargins}.

\begin{table}
\begin{centering}
\begin{footnotesize}
\begin{tabular}{cc|cc}
\multicolumn{2}{l|}{\dtname{Paleo}} & \multicolumn{2}{c}{M} \\
&& N & S\\
\hline
\multirow{2}{*}{IM} & N & 8832 & 145\\
& S & 60 & 544\\
\end{tabular}
\hspace{7mm}
\begin{tabular}{cc|cc}
\multicolumn{2}{l|}{\dtname{Paleo}} & \multicolumn{2}{c}{CM} \\
&& N & S\\
\hline
\multirow{2}{*}{IM} & N & 8977 & 0\\
& S & 614 & 0\\
\end{tabular}\\
\vspace{7mm}
\begin{tabular}{cc|cc}
\multicolumn{2}{l|}{\dtname{Courses}} & \multicolumn{2}{c}{M} \\
&& N & S\\
\hline
\multirow{2}{*}{IM} & N & 402 & 47\\
& S & 41 & 330\\
\end{tabular}
\hspace{6.5mm}
\begin{tabular}{cc|cc}
\multicolumn{2}{l|}{\dtname{Courses}} & \multicolumn{2}{c}{CM} \\
&& N & S\\
\hline
\multirow{2}{*}{IM} & N & 437 & 6\\
& S & 326 & 51\\
\end{tabular}\hspace{.5mm}
\end{footnotesize} 
\par\end{centering}
\caption{Some of the contingency tables of the significance of itemsets in randomizations \prname{Margins} (M), \prname{ItemsetMarginsSoft} (IM) and \prname{ClusterMargins} (CM). S represents statistical significance and N the opposite. All pairs of columns in a data set were used as itemsets.\label{tab:pairs_numbers}}
\end{table}

When comparing the contingency table of \prname{ItemsetMarginsSoft} and \prname{Margins} in \dtname{Paleo}, it is evident that many of the itemsets were not significant in either randomization. Still, several itemsets were statistically significant in both cases. One example of such an itemset with two species is \emph{Brachyodus} and \emph{Bunolistriodon}. Both were contained in separate itemsets constrained in \prname{ItemsetMarginsSoft}. Apparently, having both species in separate itemset constraints did not require to maintain the support of their combined itemset, and hence, it was statistically significant.

Other interesting itemsets are the ones that were statistically significant in \prname{ItemsetMarginsSoft} but not in \prname{Margins}. One example of such itemset is \emph{Hyainailouros} and \emph{Amphimoschus}. Again both species were contained in separate itemset in the constraints, but their combined itemset was not. These constraints had a clear decreasing effect in the support of the itemset, and therefore, it was found significant when randomizing with \prname{ItemsetMarginsSoft}, but not with \prname{Margins}.

\looseness-1 
The contingency table between \prname{ClusterMargins} and \prname{ItemsetMarginsSoft} shows that no itemset was statistically significant when constraining the randomization with clustering. Even though the \dtname{Paleo} data set is known to have three clear clusters, the extent of this constraint was surprising. We expected to see at least few itemsets significant in \prname{ClusterMargins}. Evidently the clustering result fully explains the pairwise correlations of the species in the data set.
 
Similar results were found from \dtname{Courses}. An example of a statistically significant itemset in \prname{Margins} and \prname{ItemsetMarginsSoft} is \emph{Computer Organization} and \emph{Database Systems I}. This was expected, since 90\% of students who took \emph{Database Systems I} also took \emph{Computer Organization}.

One example of a course pair that was significant in \prname{ItemsetMarginsSoft} but not significant in \prname{Margins} is \emph{Models for Programming and Computing} and \emph{Introduction to the Use of Computers}. One possible reason for this is that this pair of courses is taken by very different kind of people. The itemsets of these groups introduce anti-correlation between the courses in the example, and therefore, that itemset becomes statistically significant in \prname{ItemsetMarginsSoft}.

The clustering results are again similar to the ones in \dtname{Paleo}, with the exception that now some itemsets are significant in \prname{ClusterMargins}. An example of a pair of courses that is significant in \prname{ItemsetMarginsSoft} and \prname{ClusterMargins} is \emph{Reading Comprehension in English} and \emph{English Oral Test}. These courses are not specific to any study programme in computer science, and therefore, are not affected by clustering. The clustering mostly separates students with respect to computer science courses, which makes these general courses poorly clustered. However, this allows them to be statistically significant even when constraining the randomization with the clustering.

\subsection{Using Significant Itemsets as Constraints}

The motivation of the second experiment is to simulate an actual data mining scenario, where the patterns are discovered by a data mining algorithm and these are then used to assess which are significant in different randomization scenarios. Essentially, the idea is to find out the possible interrelations between itemsets and clustering. 

We focus on frequent sets of size 2 and 3 mined from the used data sets. Here we use again the same randomizations \prname{Margins}, \prname{ItemsetMarginsSoft} and \prname{ClusterMargins}, but with different constraints for \prname{ItemsetMarginsSoft}.

The first set of constraints in \prname{ItemsetMarginsSoft} is constructed by adding the 40 most significant itemsets from \prname{Margins}, i.e., the ones with the smallest $p$-values. The interest is how the significance of the other itemsets change when constraining with the most significant itemsets.

The second set of itemset constraints is used for \prname{ItemsetMarginsSoft}, but now the set contains the itemsets that had the largest increase in their $p$-value from \prname{Margins} to \prname{ClusterMargins}. The intuition is that these itemsets may explain the clustering since they were explained by it. We use ICM for short for this randomization.

Table~\ref{tab:sig_numbers} depicts a few interesting contingency tables of the significance of itemsets in two randomizations at a time.  The clustering was statistically significant in \prname{Margins}, \prname{ItemsetMarginsSoft} and ICM, but not in \prname{ClusterMargins}.

\begin{table}
\begin{centering}
\begin{footnotesize}
\begin{tabular}{cc|cc}
\multicolumn{2}{l|}{\dtname{Paleo}} & \multicolumn{2}{c}{CM} \\
&& N & S\\
\hline
\multirow{2}{*}{IM} & N & 1882 & 0\\
& S & 687 & 0\\
\end{tabular}
\hspace{7mm}
\begin{tabular}{cc|cc}
\multicolumn{2}{l|}{\dtname{Paleo}} & \multicolumn{2}{c}{ICM} \\
&& N & S\\
\hline
\multirow{2}{*}{IM} & N & 1572 & 306\\
& S & 162 & 535\\
\end{tabular}\\
\vspace{7mm}
\begin{tabular}{cc|cc}
\multicolumn{2}{l|}{\dtname{Courses}} & \multicolumn{2}{c}{CM} \\
&& N & S\\
\hline
\multirow{2}{*}{IM} & N & 297 & 13\\
& S & 809 & 146 \\
\end{tabular}
\hspace{6.5mm}
\begin{tabular}{cc|cc}
\multicolumn{2}{l|}{\dtname{Courses}} & \multicolumn{2}{c}{ICM} \\
&& N & S\\
\hline
\multirow{2}{*}{CM} & N & 265 & 841\\
& S & 2 & 157\\
\end{tabular}\hspace{.5mm}
\end{footnotesize} 
\par\end{centering}
\caption{Some of the contingency tables of the significance of itemsets in randomizations \prname{Margins} (M), \prname{ItemsetMarginsSoft} (IM), \prname{ClusterMargins} (CM) and \prname{ItemsetMarginsSoft} constrained by itemsets found using \prname{ClusterMargins} (ICM). S represents statistical significance and N the opposite. Frequent itemsets of size 2 and 3 were used as itemsets.\label{tab:sig_numbers}}
\end{table}

The \dtname{Paleo} results between \prname{ItemsetMarginsSoft} and \prname{ClusterMargins}  are very similar to the previous section. None of the itemsets were significant in \prname{ClusterMargins}, and a portion was significant in \prname{ItemsetMarginsSoft}. The strength of the clustering result clearly also affect the itemsets of size 3.  
 
The contingency table between ICM and \prname{ItemsetMarginsSoft} displays the fact that neither of the constraint sets completely explain the data since roughly 500 itemsets were significant in both, and few hundred in either. Still, 60\% of the itemsets were explained by both constraint sets along with the margins. However, 1419 itemsets were never significant even with \prname{Margins}, which means that the expressive power of both constraint sets is very limited.

\prname{ClusterMargins} is very different from ICM. Although we tried to construct the set of itemsets to constrain in ICM in such a way that the results would be the same as with \prname{ClusterMargins}, the numbers of significant itemsets do not show this. 

The results for \dtname{Courses} display similar behavior between \prname{ItemsetMarginsSoft} and \prname{ClusterMargins} as in previous experiment, and different from \dtname{Paleo}. Clearly, the clustering of \dtname{Courses} does not describe the data as well as for \dtname{Paleo}. One example of an itemset of courses significant in both is \emph{Programming Project} and \emph{Programming in Java}. Evidently, these have much in common and are most likely taken both instead of just either of them. 

The comparison between \prname{ClusterMargins} and ICM expresses the same as with \prname{ClusterMargins} and \prname{ItemsetMarginsSoft}. We can conclude also here that ICM with these constraints is not as strict, and conversely, does not explain as much, as \prname{ClusterMargins}. An example itemset significant in ICM but not in \prname{ClusterMargins} is \emph{Information Systems Project}, \emph{Programming Project} and \emph{Data Communications}. The last two was an itemset constrained in ICM, and the first was also a part of a separate constraint. 

This itemset had high frequency in some clusters but very low in others. Because of this, the randomization had little room to break the itemset in \prname{ClusterMargins}. We conjecture that these courses are most likely all required in some study programme. Additionally, \emph{Programming Project} and \emph{Data Communications} together did not explain the occurrence of the triplet, and therefore it was significant in ICM. 

\subsection{Discovering Significant Itemsets Iteratively}

In our final experiment we conduct an iterative data mining process in which
the itemsets are added iteratively to constraints. We used again the frequent
sets of size 2 and 3. 

\looseness-1
First, the data is randomized with \prname{Margins} and the itemset with the smallest $p$-value is inserted to the set of itemset constraints. The data is then randomized at each iteration with \prname{ItemsetMarginsSoft} and always the itemset with the smallest $p$-value is added to the set of constraints. This resembles the situation where the user iteratively tries to understand one pattern at a time and 
wants to find which
patterns are not explained by the already understood patterns. 

We carried out a total of 10 iterations. Figure~\ref{fig:iter_numbers} displays the number of itemsets found significant at each iteration, including the initial \prname{Margins} randomization.

\begin{figure}
\begin{centering}
\raisebox{1cm}{\rotatebox{90}{\fontfamily{phv}\selectfont{\scriptsize\# significant patterns}}} 
\hspace{1mm}
\subfigure[\dtname{Paleo}]{\includegraphics{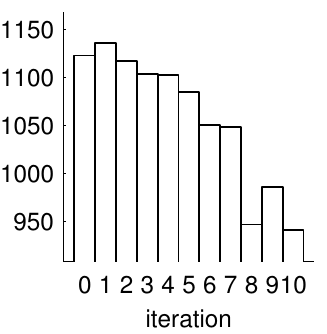}}%
\hspace{3mm}
\subfigure[\dtname{Courses}]{\includegraphics{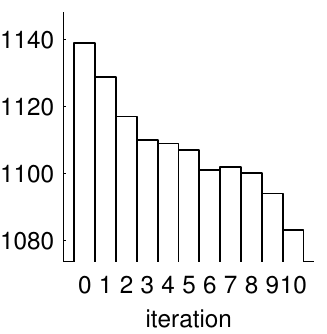}}
\par\end{centering}
\caption{The number of significant itemsets at each iteration. Notice the different vertical scale and that the lower parts of the bars have been cropped to better see the difference between iterations.\label{fig:iter_numbers}}
\end{figure}

The results follow almost with no fail the intuition that when constraints are added, the number of significant itemsets decreases. However, this may not always be the case, as seen in the previous experiments. Sometimes adding constraints increases the statistical significance of some patterns by introducing anti-correlation. Still, the intuitive results may be due to how the itemset constraints were selected. At any time, the constrained itemsets do not explain the itemsets found statistically significant. However, the significant itemsets are still likely to have some correlation between them, and adding one of them to the set of constraints will restrict the rest. Selecting itemsets in another fashion may produce very different results.

\section{Conclusions}\label{sec:conclusions}

Our focus in the paper was to study the concept of iterative data
mining. The idea behind this approach is the question whether the
results of one analysis explains or implies the results of another
analysis. This approach can be then refined into iterative data mining
process in which the user iteratively selects interesting patterns or
models that are then used for updating the significance of the rest of
the patterns or models. 

Our approach is to produce random data sets having the same selected
statistics as the original dataset. As constraining statistics we used
row and column margins, itemsets, and clustering structure. Using these
random data sets we computed empirical $p$-values for our test
statistics. Our experiments demonstrated that our method works in
practice. 
 

\end{document}